\documentclass[11pt,a4paper]{article}

\usepackage[margin=2.5cm]{geometry}
\usepackage{amsmath,amssymb,amsthm}
\usepackage{booktabs}
\usepackage{graphicx}
\usepackage{hyperref}
\usepackage{xcolor}
\usepackage{algorithm}
\usepackage{algpseudocode}
\usepackage{multirow}
\usepackage{array}
\usepackage{enumitem}
\usepackage{microtype}
\usepackage{tikz}
\usetikzlibrary{shapes.geometric,arrows.meta,positioning,fit,backgrounds,calc}

\hypersetup{
  colorlinks=true,
  linkcolor=blue!60!black,
  citecolor=blue!60!black,
  urlcolor=blue!60!black,
}

\theoremstyle{definition}
\newtheorem{definition}{Definition}[section]
\theoremstyle{plain}
\newtheorem{theorem}[definition]{Theorem}
\newtheorem{proposition}[definition]{Proposition}

\theoremstyle{remark}

\title{\textbf{Quantum-Inspired Trace-Augmented Evidence Selection for Reasoning over Structured Hypothesis Spaces}}

\author{
 Laura Wynter\thanks{Corresponding author: lwynter@smu.edu.sg}\, , Nirvik Sahoo, Paul Griffin \\ School of Computing and Information Systems\\
 Singapore Management University\\Singapore
}

\date{}

\begin{document}

\maketitle

\begin{abstract}
Large language models (LLMs) now solve a wide range of expert-level exams at
or above human level, yet remain brittle on specialised, evidence-intensive
domains such as law. On these tasks, errors arise not only from gaps in world
knowledge but also from subtle distinctions between pieces of evidence and
inconsistent use of supporting evidence. The most common aggregator over
sampled chain-of-thought (CoT) traces, majority vote, returns the most popular
answer regardless of whether its evidence is actually strongest.
We propose to treat the selection of CoT reasoning fragments into a set of
evidence as an explicit combinatorial optimisation problem, allowing
well-supported but minority hypotheses to override noisy majorities, and to
evaluate the approach on legal-reasoning benchmarks that are particularly
sensitive to evidence quality.
We introduce EP-HUBO (Evidence Pool Higher-Order Binary Optimisation), which
(i)~generates multiple CoT traces with a small local model, (ii)~parses
fragments into per-hypothesis evidence pools, (iii)~solves a higher-order
unconstrained binary optimisation per pool with quality-derived weights
(relevance, specificity, distinctiveness), and (iv)~delegates a single
adjudication call per question to a frontier model. We evaluate EP-HUBO on
two evidence-intensive legal benchmarks (MMLU-Pro law and LEXam) using both
simulated annealing on classical hardware and the Dirac-3 photonic
entropy-quantum machine from Quantum Computing Inc.
EP-HUBO beats majority vote by $+12.6$~pp on MMLU-Pro law and by up to
$+23.2$~pp on LEXam with a strong frontier adjudicator, and beats zero-shot
frontier adjudication by $+1.5$~pp on MMLU-Pro law and up to $+5.1$~pp on
LEXam. On LEXam, zero-shot Claude Sonnet~4.6 exhibits a severe position bias,
choosing option ``E'' on 87.7\% of questions; EP-HUBO-selected evidence
reduces the bias and yields up to $+20.2$~pp over zero-shot Sonnet at HUBO
precision $92.0\%$. Quantum solutions on Dirac-3 are competitive with
classical simulated annealing.
HUBO-style optimisation gives a principled way to aggregate reasoning
fragments while preserving minority-but-correct hypotheses, and is most
valuable in low-contamination domains where frontier models have not already
absorbed the benchmark material. 
\end{abstract}

\section{Introduction}

Large language models (LLMs) now solve a wide range of expert-level
exams at or above human level, yet they remain brittle on specialised,
evidence-intensive domains such as law.
On these tasks, errors arise not only from gaps in world
knowledge, but also from  subtle distinctions between pieces of evidence
and inconsistent use of supporting evidence.  A common approach is
to sample many LLM responses and aggregate the answers.  The
aggregator most commonly used, majority
vote~\cite{wang2022selfconsistency}, returns the most
frequent answer regardless of whether its evidence is actually
strongest. 

We introduce EP-HUBO: Evidence Pool Higher-Order Binary
Optimisation,  that treats evidence selection in Chain-of-Thought  (CoT) reasoning traces as a
structured combinatorial problem. 
Unlike self-consistency and majority-vote schemes,
EP-HUBO does not reward popularity: reasoning fragment weights are derived from relevance,
specificity, and distinctiveness, allowing well-supported but minority hypotheses to override
noisy majorities. 
Our approach involves using the optimisation problem to identify the strongest CoT reasoning evidences  from per-hypothesis evidence pools over the structured hypothesis space.
EP-HUBO is a quantum-inspired formulation that allows solving the optimisation problem with
either a classical or a quantum computer. 

EP-HUBO is characterised by: (i) local open-weight trace generation, (ii) per-hypothesis evidence pools, (iii) HUBO with quality-derived rather than popularity-derived weights, and (iv) a single frontier adjudication call per question.
The approach has four phases.  First, for cost efficiency, we use a smaller 
 local  model to generate multiple CoT traces per question. Then, also using the smaller local model along with deterministic rules, fragments from the traces are parsed into answer-specific
      evidence pools.   This step relies on the hypothesis space being well-structured, so that  the answers can be  partitioned into a discrete set
of candidate hypotheses.
Third, we formulate and solve a higher-order unconstrained binary optimisation (HUBO) problem; the optimisation problem serves to select a  subset of fragments that supports strongly each answer pool. Lastly, 
 a single call to a larger, frontier model adjudicates
      over the options using the optimised evidence. 

The legal domain presents an excellent area of application for EP-HUBO. Legal decisions rely on compiling separate, often independent, pieces of evidence to support a conclusion. Contrary to mathematical proofs, the evidences need not follow a specific order. As such, they may be drawn from multiple, different CoT reasoning traces. In addition, if a reasoning trace misses key pieces of evidence, the legal reasoning may collapse, which motivates generating multiple CoT traces. 
We thus evaluate EP-HUBO on two challenging legal reasoning benchmarks:
the law subset of MMLU-Pro~\cite{wang2024mmlupro} 
and the LEXam benchmark~\cite{lexam2024} of Swiss and
international-law questions. We evaluate EP-HUBO on both classical computers and a quantum computer; using classical computers we solve the optimisation problem via simulated annealing (SA), and we solve the quantum formulation on the Dirac-3 photonic entropy quantum machine from Quantum Computing Inc. (QCi) ~\cite{Nguyen2025EntropyComputing}. Our results show that EP-HUBO substantially improves over majority vote on both
benchmarks.  Beyond accuracy, EP-HUBO  exposes and  mitigates adjudicator
bias. On LEXam, one of the frontier LLMs exhibits a severe
 bias, choosing one of the answer options  on 87.7\% of all questions. However, using HUBO-selected evidence reduces the bias and yields an $+11.4$~pp gain over the
zero-shot LLM.  Using Dirac-3 to solve the HUBO optimisation gives comparable results to simulated annealing on classical computers.

 This paper thus makes the following contributions.

\begin{enumerate}\setlength\itemsep{2pt}
  \item Our evidence-pool HUBO framework. 
       We define the trace-augmented
       problem formulation with evidence pools, and a higher-order binary optimisation problem
      for evidence selection, which is amenable to both classical and quantum computers.
  \item A theoretical analysis of EP-HUBO.       
  \item An empirical study using both classical and a photonic quantum computer on  MMLU-Pro law and
          LEXam
        Swiss and international-law benchmarks using two open-weight smaller
     LLM  trace generators. 
\end{enumerate}

The paper is organised as follows. Section~\ref{sec:related} discusses related work including self-consistency, combinatorial
optimisation for NLP, quantum-inspired AI, and benchmark contamination.
Section~\ref{sec:framework} introduces our  framework and
algorithm.  Section~\ref{sec:theory} provides our theoretical
analysis.  Section~\ref{sec:results} describes our experimental setup and reports our results.
Section~\ref{sec:ablation} presents an ablation study that isolates the contribution of EP-HUBO's method
components and the effect of adjudicator model strength.
Section~\ref{sec:conclusion} concludes along with a discussion of the implications and limitations of our work.
All scripts and traces are released alongside this paper.

\section{Related Work}
\label{sec:related}
First, we discuss prior work   on CoT reasoning traces as well as other combinatorial and quantum approaches.

\paragraph{Self-Consistency and Majority Vote}

Self-consistency~\cite{wang2022selfconsistency} is an approach that generates $N$ independent CoT
traces for a single task. It is used in general with majority vote, that  selects the most-common final answer.  While it improves over greedy
decoding, it  is limited by the majority signal: if the correct answer is in the
minority of traces self-consistency cannot recover it.  Recent theoretical analysis
by Feng et al.~\cite{feng2025optimalsc} shows that its accuracy
follows a power-law in the number of samples; they introduce an adaptive
variant (Blend-ASC) that achieves the same accuracy with 6.8$\times$ fewer
samples by detecting when additional traces yield diminishing returns.
Kang et al.~\cite{kang2025selfcertainty} propose \emph{self-certainty}, a
reward-model-free metric derived from the model's own output distribution,
enabling sample-efficient best-of-$N$ selection without external verifiers.

\paragraph{Reasoning verification and generative verifiers.}
Beyond self-consistency, Yao et
al.~\cite{yao2023tot} introduce Tree-of-Thought search over partial
solutions; Saunders et al.~\cite{saunders2022selfcritique} show that
language models can be fine-tuned to critique their own intermediate
steps; Welleck et al.~\cite{welleck2022generating} train a separate
generative verifier whose score is used to re-rank candidate answers.
These methods improve over self-consistency by introducing a
verification stage, but the verification scope is typically the
\emph{full reasoning trace}, not the individual evidence fragments
within it.  EP-HUBO can be seen as taking the verification idea one
level finer: rather than ranking complete traces, we score and select
fragments and let the frontier model verify the resulting evidence
ensemble.

\paragraph{Combinatorial structure in NLP decoding.}
Beam search is itself a discrete
optimisation: each step selects from a candidate set under a
log-probability objective.  More structured combinatorial decoding has
been studied at least since Roth and Yih~\cite{rothyih2004ilp}, who
formulated semantic-role labelling as an integer linear program, and
Riedel and Clarke~\cite{riedelclarke2006ilp} who used ILP for
dependency parsing.  Modern variants use SAT solvers for constrained
generation~\cite{poesia2022synchromesh} or differentiable relaxations
of structured prediction objectives~\cite{niculae2018sparsemap}.
EP-HUBO continues this tradition by treating evidence-fragment
selection as a higher-order pseudo-Boolean optimisation, with the
distinguishing feature that the objective coefficients are produced by
a learned model rather than designed by hand.

\paragraph{Scaling Test-Time Compute and Reinforcement-Learned Reasoning}

A parallel line of research improves reasoning by scaling inference-time
computation.  DeepSeek-R1~\cite{deepseekr1}  demonstrated that reinforcement
learning without supervised chain-of-thought annotation induces emergent
behaviours---self-reflection, verification, and strategy adaptation---achieving
84.0\% on MMLU-Pro.  Yang et
al.~\cite{yang2025tops} show that na\"ively extending CoT length can
\emph{hurt} accuracy on certain problem types and propose a
\emph{Thinking-Optimal Scaling} strategy that lets models self-select minimal
reasoning length per problem.  These findings motivate our use of bounded number of 
structured traces rather than unbounded chain-of-thought.

\paragraph{Multi-LLM Aggregation and Mixture of Agents}

Wang et al.~\cite{wang2024moa} propose Mixture-of-Agents (MoA), a layered
architecture in which each agent refines outputs from all prior agents,
achieving substantial gains over single-model baselines on open-ended
generation tasks.  A follow-up by Li et al.~\cite{li2025selfmoa} shows that
\emph{Self-MoA}---aggregating multiple outputs from the single strongest
model---outperforms diverse multi-model mixtures on AlpacaEval~2.0, suggesting
that answer diversity from one capable model is more valuable than architectural
diversity across weaker models.  Ashiga et al.~\cite{ashiga2025ensemblesurvey}
survey seven ensemble paradigms for LLMs (weight merging, mixture-of-experts,
output ensembling, routing, cascading) and find that output-level ensembles
with voting are competitive with more complex fusion strategies when the
individual models are sufficiently capable.   EP-HUBO on the other hand
aggregates
\emph{evidence fragments} from multiple traces of a \emph{single model}, then
delegates the final synthesis to a stronger frontier model. Note that EP-HUBO could easily be extended to select from traces across multiple LLMs.

\paragraph{Combinatorial Optimisation of Reasoning Fragments}
The idea of casting reasoning-fragment selection as a combinatorial optimisation
problem was introduced by Esencan et al.~\cite{cr2024qubo}, who map LLM-generated
candidate reasons onto a quadratic QUBO problem to select an optimal evidence subset for
chain-of-thought prompting, benchmarked against majority voting.   Zhang et
al.~\cite{zhang2025llmqubo} extend this paradigm to an end-to-end framework
(LLM-QUBO) in which an LLM automatically parses natural-language problem
descriptions and generates QUBO formulations, integrated with a hybrid
quantum-classical Benders' decomposition. Relative to Esencan et al. and LLM-QUBO, our EP-HUBO is (i) per-hypothesis, (ii) uses small local models for both trace generation and scoring, and (iii) is evaluated systematically on the legal domain as it presents an ideal application for optimally selecting and combining CoT reasoning traces.

\paragraph{QCR-LLM}

QCR-LLM~\cite{qcr2025} extends the
formulation from a quadratic QUBO to HUBO, with third-order interactions, and applies it  to
LLM reasoning traces.  They extract atomic reasoning fragments from
$N=20$ frontier-model CoT traces per question  and encode the traces  as binary variables
$x_i \in \{0,1\}$.  The HUBO energy function is defined as:
\begin{equation}
  H(\mathbf{x}) = \sum_i w_i x_i
                + \sum_{i<j} w_{ij} x_i x_j
                + \sum_{i<j<k} w_{ijk} x_i x_j x_k
  \label{eq:hubo}
\end{equation}
where 1-body coefficients encode fragment \emph{popularity} (fraction of traces
containing the fragment):
\begin{equation}
  w_i = -\mu p_i + \alpha \cdot p_i(1-p_i)
  \label{eq:1body_qcr}
\end{equation}
and pairwise/triplet coefficients encode \emph{statistical co-occurrence}
corrected for a semantic similarity penalty.  The low-energy subset (bottom
25\textsuperscript{th} percentile) is selected and passed to a frontier model
for final inference.  Their results on BIG-Bench Extra Hard (BBEH) show gains of up to 9~pp
over GPT-4o, DeepSeek R1, and o3-high, with $\approx$5$\times$ better energy
efficiency than o3-high.

QCR-LLM is an important first step toward HUBO-based combinatorial reasoning in LLMs, but several aspects of the formulation merit further scrutiny.
(1) Trace generation cost: QCR-LLM uses frontier models (GPT-4o,
DeepSeek) to generate each of their 20 traces for each question.  For large-scale deployment over a
full benchmark (e.g.\ 12,032 MMLU-Pro questions) this incurs substantial API
cost and largely defeats the purpose of their subsequent  optimisation step. 
(2) MV-correlated weights: QCM-LLM selection is based on fragment popularity $p_i$ (eq.~\ref{eq:1body_qcr}), which
is directly proportional to the fraction of traces endorsing that fragment.
Since traces endorse fragments that support the majority-vote answer,
the QCM-LLM 1-body term implicitly replicates the MV signal. (3)  Missing baseline:  QCR-LLM reports gains against frontier
 baselines (GPT-4o, o3-high) but  does not report
majority vote across its own frontier-LLM-generated 20-trace ensemble. 

While QCR-LLM applies HUBO to a single fragment set,  our EP-HUBO selects
separately from each candidate-answer pool  then provides  selected fragments
from all pools  to a frontier model for adjudication.  Our fragments are generated by local, open-weight models, and EP-HUBO uses a frontier model only once per question for adjudication across HUBO-selected evidences.
 Our EP-HUBO optimiser  selects the best evidence \emph{for each hypothesis}
        independently, rather than selecting a consensus that may suppress
        minority-but-correct reasoning.

\paragraph{Quantum and Classical Annealing for Combinatorial AI Tasks}

 Pomeroy et al.~\cite{pomeroy2025qannealml} reformulate feature
selection, instance selection, and clustering as QUBO problems and show that
D-Wave quantum annealing matches or exceeds classical SA on standard ML
benchmarks, even on current noisy hardware.  Nausheen et
al.~\cite{nausheen2025qnlpsurvey} survey quantum natural language processing
approaches---from quantum encoding of word embeddings to task-specific
quantum circuits---identifying QUBO-based combinatorial reasoning as one of
the most practically mature quantum-NLP interfaces available today.
The standard HUBO‑to‑QUBO reduction proceeds by quadratization of cubic terms; on D‑Wave quantum hardware this induces a polynomial blowup in the number of qubits. Hardware‑aware HUBO solvers such as  the BF‑DCQO procedure used by QCR‑LLM on IBM gate‑model devices operate directly on the higher‑order formulation. In our work, Phase 3 is solved using both classical simulated annealing and the Dirac-3 photonic quantum machine. Dirac-3 allows us to work directly with the same higher-order HUBO-level formulation, avoiding the auxiliary-variable quadratization that would typically be required for direct implementation on D-Wave quantum annealers, and reducing the reformulation or circuit-overhead burden that often arises in gate-based QAOA approaches.

\section{The EP-HUBO Framework}
\label{sec:framework}

\subsection{Problem Formulation}
\label{sec:problem}

We formalise the trace-augmented  question-answering
setting that EP-HUBO addresses. Table \ref{tab:notation} summarises the notation used throughout the
paper.

\begin{table}[h]
\centering
\small
\begin{tabular}{lp{10cm}}
\toprule
Symbol & Meaning \\
\midrule
$q,\ a^*(q)$ & Question; gold answer. \\
$\mathcal{A}, L$ & Option set, number of options. \\
$\mathcal{M}_L,\ \mathcal{M}_F$ & Local model (Phase 1--2) and frontier model (Phase 4). \\
$\mathcal{T}(q),\ N$ & Set of traces for $q$; traces per question. \\
$\tau = (r,\hat{a})$ & Single trace: reasoning text and extracted answer. \\
$P_\ell(q)$ & Evidence pool for option $a_\ell$ (fragments from traces with $\hat{a}=a_\ell$). \\
$\mathcal{L}(q)$ & Set of populated pools (with $\ge \tau_{\min}=3$ traces). \\
$\mathbf{x}_\ell \in \{0,1\}^{n_\ell}$ & Binary selection indicator on $P_\ell$, $\|\mathbf{x}_\ell\|_1=K_\ell$. \\
$H_\ell(\mathbf{x}_\ell)$ & Per-pool HUBO energy (Eq.~\ref{eq:hubo}). \\
$S_\ell$ & Selected fragment subset for pool $\ell$. \\
$\hat{a}_H(q),\ \hat{a}_{ZS}(q),\ \hat{a}_{MV}(q)$ & EP-HUBO, zero-shot, majority-vote predictions. \\
$\mathrm{prec}_H$ & HUBO precision over ZS (Eq.~\ref{eq:hubo_precision_def}). \\
$\alpha,\beta,\gamma$ & 1-body coefficients (relevance, specificity, distinctiveness). \\
$\lambda_{\text{supp}}, \lambda_{\text{contra}}, \lambda_{\text{coh}}, \lambda_{\text{und}}$ & Pairwise/triplet coefficients. \\
\bottomrule
\end{tabular}
\caption{Notation used throughout the paper.}
\label{tab:notation}
\end{table}

\begin{definition}[Trace-augmented instance]
\label{def:mcqa}
A \emph{trace-augmented  instance} is a tuple
$\mathcal{I} = (q, \mathcal{A}, \mathcal{M}_L, \mathcal{M}_F, N)$ where
$q$ is a natural-language question, $\mathcal{A} = \{a_1, \dots, a_L\}$
is a finite set of answer options, $\mathcal{M}_L$ is a local language model used for trace
generation and weight scoring, $\mathcal{M}_F$ is a frontier model used
for adjudication, and $N$ is the number of independent traces generated
by $\mathcal{M}_L$ per question.  The gold answer is denoted $a^*(q)$.
\end{definition}

\begin{definition}[Trace, fragment, answer label]
\label{def:trace}
A \emph{trace} for question $q$ is a pair $\tau = (r, \hat{a})$ where
$r$ is a chain-of-thought text and $\hat{a} \in \mathcal{A}$ is the
trace's extracted final answer.  An \emph{extraction operator}
$\mathrm{frag}: r \mapsto \{f_1, \dots, f_m\}$ produces a set of
candidate evidence fragments from the reasoning text.  We write
$\mathcal{T}(q) = \{\tau_1, \dots, \tau_N\}$ for the trace set generated
for $q$.
\end{definition}

\begin{definition}[Evidence pool]
\label{def:pool}
For each answer option $a_\ell \in \mathcal{A}$, the \emph{evidence
pool} $P_\ell(q)$ is the multiset of fragments extracted from traces
that concluded with answer $a_\ell$:
$$
P_\ell(q) \;=\; \bigcup_{\tau \in \mathcal{T}(q) : \hat{a}(\tau) = a_\ell} \mathrm{frag}(r(\tau)).
$$
A pool is \emph{populated} if $|\mathcal{T}_\ell(q)| \ge \tau_{\min}$
for a minimum trace count $\tau_{\min}$; we
let $\mathcal{L}(q) \subseteq \mathcal{A}$ denote the set of options
with populated pools. We set $\tau_{\min} = 3$ in our experiments.
\end{definition}

\begin{definition}[HUBO selection]
\label{def:hubo}
For each evidence pool $P_\ell(q)$ with $|P_\ell| = n_\ell$
fragments, the EP-HUBO optimisation problem determines the  optimal binary vector
$\mathbf{x}_\ell \in \{0,1\}^{n_\ell}$ that minimises the following higher-order  energy objective function, subject to $\|\mathbf{x}_\ell\|_1 = K_\ell$, where $K_\ell$  is a cardinality hyperparameter chosen based on fragment diversity.
\begin{equation}
H_\ell(\mathbf{x}_\ell)
\;=\;
\sum_{i \in P_\ell} w_i^{(1)} x_i
\;+\;
\sum_{i < j \in P_\ell} w_{ij}^{(2)} x_i x_j
\;+\;
\sum_{i < j < k \in P_\ell} w_{ijk}^{(3)} x_i x_j x_k,
\label{eq:hubo}
\end{equation}
where  weights $w^{(1)}, w^{(2)}, w^{(3)}$ are scored by $\mathcal{M}_L$
along quality dimensions including relevance, specificity, distinctiveness, 
pairwise criteria like support vs. contradict and triplet coherence vs. undermining the answer.
Let $S_\ell = \{i \in P_\ell : x_i = 1\}$ denote the selected
fragment subset of pool $\ell$.
\end{definition}

\begin{definition}[Adjudication function]
\label{def:adjudication}
The adjudicator is a map
$
\mathcal{A}_F : (q, \mathcal{A}, \{S_\ell\}_{\ell \in \mathcal{L}(q)}) \mapsto \hat{a} \in \mathcal{A}
$
implemented as a single call to the frontier LLM $\mathcal{M}_F$ that receives $q$, all
options, and the per-pool selected fragments $\{S_\ell\}$ labelled by
their candidate answer label.  We denote the EP-HUBO output by
$\hat{a}_H(q)$.  
\end{definition}

\begin{definition}[Baselines] We use two baselines: the
\emph{zero-shot} prediction $\hat{a}_{ZS}(q) = \mathcal{A}_F(q, \mathcal{A}, \emptyset)$
which calls the frontier LLM without any evidence pools,  and the \emph{majority-vote} prediction
$\hat{a}_{MV}(q) = \arg\max_\ell |\mathcal{T}_\ell(q)|$ that uses the local LLM traces and selects the majority answer from them.
\end{definition}

\begin{definition}[HUBO precision]
\label{def:precision}
For a question distribution $\mathcal{D}$ and a fixed adjudicator
$\mathcal{M}_F$, define the win and hurt events
$W(q) = \mathbf{1}[\hat{a}_H(q) = a^*(q) \wedge \hat{a}_{ZS}(q) \ne a^*(q)]$
and
$L(q) = \mathbf{1}[\hat{a}_H(q) \ne a^*(q) \wedge \hat{a}_{ZS}(q) = a^*(q)]$.
The \emph{HUBO precision} over zero-shot is
\begin{equation}
\mathrm{prec}_H \;=\; \frac{\Pr_\mathcal{D}[W(q) = 1]}{\Pr_\mathcal{D}[W(q) = 1] + \Pr_\mathcal{D}[L(q) = 1]}
\;=\; \frac{|W|}{|W| + |L|}
\label{eq:hubo_precision_def}
\end{equation}
when estimated from a finite sample.  A value of $\mathrm{prec}_H > 1/2$
indicates that, conditional on disagreement between HUBO and ZS, HUBO
is more often correct.
\end{definition}

\subsection{Pipeline Overview}
\label{sec:pipeline_overview}

EP-HUBO proceeds in four phases as described below. See Figure~\ref{fig:pipeline}.

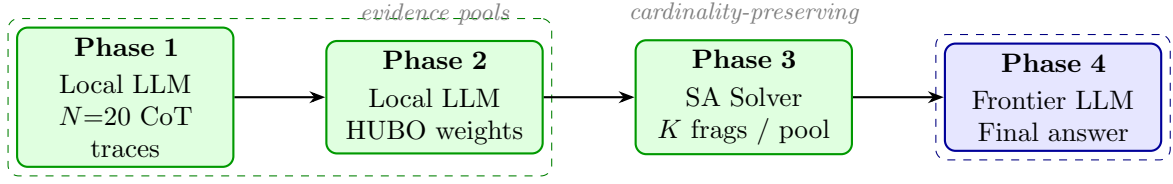
\begin{figure}[t]
\centering
\begin{tikzpicture}[
  node distance = 7mm and 12mm,
  box/.style = {rectangle, rounded corners=4pt, draw, thick,
                minimum width=28mm, minimum height=14mm,
                text width=26mm, align=center, font=\small},
  greenbox/.style = {box, fill=green!12, draw=green!60!black},
  bluebox/.style  = {box, fill=blue!10,  draw=blue!60!black},
  arr/.style      = {-{Stealth[length=6pt]}, thick},
  label/.style    = {font=\footnotesize\itshape, text=gray},
]

\node[greenbox] (p1) {\textbf{Phase 1}\\[2pt]Local LLM\\$N{=}20$ CoT traces};
\node[greenbox, right=of p1] (p2) {\textbf{Phase 2}\\[2pt]Local LLM\\HUBO weights};
\node[greenbox,  right=of p2] (p3) {\textbf{Phase 3}\\[2pt]SA Solver\\$K$ frags / pool};
\node[bluebox,  right=of p3] (p4) {\textbf{Phase 4}\\[2pt]Frontier LLM\\Final answer};

\draw[arr] (p1) -- (p2);
\draw[arr] (p2) -- (p3);
\draw[arr] (p3) -- (p4);

\begin{scope}[on background layer]
  \node[draw=green!50!black, dashed, rounded corners,
        fit=(p1)(p2), inner sep=3pt,
        label={[label,green!50!black]below:Free (local compute)}] {};
  \node[draw=blue!50!black,  dashed, rounded corners,
        fit=(p4), inner sep=3pt,
        label={[label,blue!50!black]below:One API call}] {};
\end{scope}

\node[label, above=2pt of p2] {evidence pools};
\node[label, above=2pt of p3] {cardinality-preserving};

\end{tikzpicture}
\caption{EP-HUBO four-phase pipeline.  Phases 1--3 run  locally
 (free); Phases 1-2 call a local LLM while Phase 3 uses  python code. Phase 4 requires a single frontier API call per question.
Green = local; blue = frontier LLM API.}
\label{fig:pipeline}
\end{figure}

\paragraph{Phase 1: Local LLM Trace Generation}
We use a local open-weight LLM for reasoning trace generation.  Each trace
is prompted with a standardised template requesting 3--5 key facts followed by
\texttt{FINAL ANSWER: [label]}.  

\paragraph{Phase 2: Answer-Evidence Pools}
 Traces are grouped by their extracted
answer. For each pool  with answer candidate $\ell$, fragments are
extracted from the raw \texttt{raw\_text} field using several  quality filters
(for ex. minimum 10 words; $\geq 2$ capitalised entity words, etc).  This ensures that
the HUBO operates within a hypothesis-specific evidence space: fragments in
pool $\ell$ are those produced by traces that concluded the answer is $\ell$.

 Each fragment $f_i \in \text{pool}_\ell$
is then scored by the local LLM on three dimensions:
\begin{equation}
  w_i^{(1)} = -\bigl(\alpha \cdot r_i + \beta \cdot s_i + \gamma \cdot d_i\bigr)
  \label{eq:1body_ours}
\end{equation}
where $r_i$ is \emph{relevance}   to the question,
$s_i$ is \emph{specificity} (concrete fact, statute, or named entity),
and $d_i$ is \emph{distinctiveness}, i.e., the logical contribution not implied by
other fragments.  All scores $\in [0,1]$ are elicited via a structured JSON
prompt.  No popularity or vote-frequency signal enters
Eq.~\ref{eq:1body_ours} so that the weights reflect fragment quality, not
majority-vote.

 Pairwise $(i,j)$ interactions are scored
on \emph{support} and \emph{contradict}; triplet $(i,j,k)$ interactions on
\emph{coherent} and \emph{undermining}.  Top pairs  are selected by composite
1-body score and top triplets by 2-body score.
\begin{align}
  w_{ij}^{(2)} &= -\lambda_\text{supp}\cdot\text{support}_{ij}
                  + \lambda_\text{contra}\cdot\text{contradict}_{ij} \\
  w_{ijk}^{(3)} &= -\lambda_\text{coh}\cdot\text{coherent}_{ijk}
                  + \lambda_\text{und}\cdot\text{undermine}_{ijk}
\end{align}

\paragraph{Phase 3: Simulated Annealing Selector}
The per-pool HUBO of Eq.~\ref{eq:hubo} is minimised by simulated annealing (SA) implemented in python
(see Algorithm~\ref{alg:sa}).  Each step proposes dropping one selected
fragment and adding one unselected fragment, preserving
$\|\mathbf{x}\|_1 = K$.  The quadratic $\Delta E$ is computed in $O(K)$
per step via a pre-built adjacency list keyed by selected fragment
indices; the cubic correction uses a triplet lookup table cached at
weight-construction time. By using swap moves that preserve $\|\mathbf{x}\|_1 = K$, SA explores only fixed-cardinality subsets, which regularises evidence-set size.

\begin{algorithm}[h]
\caption{Cardinality-Preserving Swap-Based Simulated Annealing for EP-HUBO}
\label{alg:sa}
\begin{algorithmic}[1]
\Require Pool $P_\ell$, weights $w^{(1)}, w^{(2)}, w^{(3)}$, cardinality $K$,
         initial temperature $T_0$, cooling rate $\eta \in (0,1)$, steps $M$.
\Ensure Selected fragment subset $S_\ell \subseteq P_\ell$ with $|S_\ell| = K$.
\State $S \gets$ greedy top-$K$ by 1-body score (warm start)
\State $E \gets H_\ell(\mathbf{1}_S)$ \Comment{evaluate Eq.~\ref{eq:hubo}}
\State $T \gets T_0$
\For{$t = 1, \dots, M$}
  \State sample $i \in S,\ j \in P_\ell \setminus S$ uniformly
  \State $S' \gets (S \setminus \{i\}) \cup \{j\}$
  \State $\Delta E \gets \Delta_{\text{1-body}}(i,j) + \Delta_{\text{quad}}(i,j;S) + \Delta_{\text{cubic}}(i,j;S)$
  \If{$\Delta E < 0$ \textbf{or} $u \sim \mathrm{Unif}(0,1) < \exp(-\Delta E / T)$}
    \State $S \gets S'$;\ $E \gets E + \Delta E$
  \EndIf
  \State $T \gets \eta T$
\EndFor
\State \textbf{return} $S$
\end{algorithmic}
\end{algorithm}

 $K$ is chosen
per pool based on intra-pool fragment diversity:
\begin{equation}
  K = \begin{cases}
    \max(2, K_\text{base}-1) & \text{if } \bar{d}_\text{Jaccard} < 0.30 \\
    K_\text{base}           & \text{if } 0.30 \le \bar{d}_\text{Jaccard} \le 0.60 \\
    K_\text{base}+1         & \text{if } \bar{d}_\text{Jaccard} > 0.60
  \end{cases}
\end{equation}
where $\bar{d}_\text{Jaccard}$ is the mean pairwise Jaccard distance between
fragment word sets.  In our experiments fragment diversity averaged 0.91,
consistently triggering $K=K_\text{base}+1=4$.

\paragraph{Phase 4: Frontier Adjudication}

A single frontier LLM API call receives the question, all options, and the
HUBO-selected fragments for each evidence pool, labelled by the pool answer. For example:
\begin{verbatim}
=== Evidence supporting (D) ===
  1. Under common law, force used in escape immediately after
     taking converts larceny to robbery.
  ...
=== Evidence supporting (E) ===
  1. Robbery requires force used *to* obtain the property...
\end{verbatim}
The model is instructed to reply with the answer option most strongly supported by the
evidence.  All  frontier calls are submitted as a single  Batch
API job.

We apply domain-specific $\lambda$ presets based on reasoning style. Additionally, $\lambda_\text{contra}$ is scaled up by up to $2\times$ when
cross-pool vocabulary overlap exceeds 0.50, signalling that pools are making
competing claims on the same domain concepts.
\begin{table}[h]
\centering
\begin{tabular}{lcccccc}
\toprule
Category & $\alpha$ & $\beta$ & $\gamma$ & $\lambda_\text{supp}$ & $\lambda_\text{contra}$ & $\lambda_\text{coh}$ \\
\midrule
Law      & 0.35 & 0.50 & 0.30 & 0.30 & 1.20 & 0.50 \\
\bottomrule
\end{tabular}
\caption{Category-specific HUBO hyperparameters.  Law questions
emphasise specificity and penalise contradiction heavily (legal facts are
binary).  }
\label{tab:presets}
\end{table}

\section{Theoretical Analysis}
\label{sec:theory}

This section establishes properties of EP-HUBO that motivate the
empirical findings.

A central design choice of EP-HUBO is to score fragments by per-fragment
\emph{quality dimensions} (relevance, specificity, distinctiveness)
rather than by frequency as is used for majority vote. 

\begin{proposition}[MV-decoupling of the HUBO weights]
\label{prop:mv_decoupling}
Assume that the per-fragment quality scores $r_i, s_i, d_i \in [0,1]$
produced by $\mathcal{M}_L$ are computed from the fragment
text $f_i$ and the question $q$ alone, without access to the pool
$P_\ell$ or to the number of traces $|\mathcal{T}_\ell(q)|$.
Then for any question $q$ and any pool $\ell$, the 1-body weight
$w_i^{(1)} = -(\alpha r_i + \beta s_i + \gamma d_i)$ is independent of
$|P_\ell|$ and $|\mathcal{T}_\ell(q)|$ in the sense that
$\mathrm{Cov}(w_i^{(1)}, |P_\ell|) = 0$ over any sampling of traces
that preserves the fragment-text marginals.
\end{proposition}

\begin{proof}
By construction the LLM scoring prompt receives only $(q, f_i)$ as
input; pool size enters neither the prompt nor the post-processing of
the score.  Hence $w_i^{(1)}$ is conditionally independent of $|P_\ell|$ given $f_i$ and covariance zero follows.
\end{proof}

Empirically, this is reflected in the large gains over MV shown in the Results section, where fragment selection clearly departs from popularity-based aggregation.
This contrasts with 
 the QCR-LLM formulation of \cite{qcr2025}, in which 1-body
coefficients $w_i^{(1)} = -\mu p_i + \alpha p_i (1-p_i)$ are explicit
functions of the empirical fragment frequency $p_i$. EP-HUBO solves $|\mathcal{L}(q)|$ independent HUBO problems, one per
populated pool, rather than a single cross-pool optimisation problem.  We next show that this decomposition
has no loss relative to joint optimisation.

\begin{proposition}[Per-pool decomposability]
\label{prop:decomp}
Suppose the adjudication function $\mathcal{A}_F$ is invariant under
permutations of the labelled-pool inputs, i.e.
$
\mathcal{A}_F(q, \mathcal{A}, \{S_{\sigma(\ell)}\}_{\ell}) = \mathcal{A}_F(q, \mathcal{A}, \{S_\ell\}_{\ell})
$
for every permutation $\sigma$ of $\mathcal{L}(q)$.  Then the joint
HUBO problem
$
\min_{\mathbf{x}_1, \dots, \mathbf{x}_L} \sum_\ell H_\ell(\mathbf{x}_\ell)
$
decomposes exactly into the $|\mathcal{L}(q)|$ independent per-pool
optimisations solved by EP-HUBO in Phase~3.
\end{proposition}

\begin{proof}
The energy objective function  $\sum_\ell H_\ell(\mathbf{x}_\ell)$ is additively
separable across pools because $H_\ell$ depends only on
$\mathbf{x}_\ell$.  Therefore any joint minimiser is a concatenation
of per-pool minimisers, and conversely.  Adjudicator permutation
invariance ensures that re-labelling pool order does not change the
downstream  decision.
\end{proof}

While permutation invariance should  always hold, we noticed that our LEXam
results in fact show a prior
in one of the zero-shot frontier models.  This is a non-standard occurrence however.
Our next result shows that EP-HUBO is at least as expressive as
majority vote: there exists a weight assignment for which the pipeline
recovers the MV prediction. 

\begin{proposition}[MV recoverability]
\label{prop:mv_recover}
Let the 1-body weights be set as $w_i^{(1)} = -|\mathcal{T}_\ell(q)|/N$
for every $i \in P_\ell$, and let all pairwise and triplet weights
vanish.  Then for any $K \ge 1$ the per-pool HUBO objective is
minimised by selecting any $K$ fragments from the largest pool, and
the resulting adjudicator input maximally weights the majority answer.
\end{proposition}

\begin{proof}
With pairwise and triplet terms zero, $H_\ell(\mathbf{x}) = \sum_i w_i^{(1)} x_i = -K |\mathcal{T}_\ell(q)|/N$
for every cardinality-$K$ selection in pool $\ell$.  Hence the minimum
energy over all $(\ell, S)$ pairs is achieved at the pool with the
largest trace count, recovering MV.
\end{proof}

\begin{theorem}[Sample complexity of HUBO precision]
\label{thm:samplecomplexity}
Let $n = |W| + |L|$ be the number of observed HUBO/ZS disagreements on
a held-out sample, and let $\hat{p}_H = |W|/n$ be the empirical HUBO
precision estimator.  Then for any $\delta \in (0,1)$, with probability
at least $1 - \delta$,
\begin{equation}
\bigl| \hat{p}_H - \mathrm{prec}_H \bigr|
\;\le\;
\sqrt{\frac{\log(2/\delta)}{2n}}.
\label{eq:hoeffding}
\end{equation}
\end{theorem}

The next result shows that when the trace distribution collapses to a
single pool, EP-HUBO cannot differ from MV.  This explains why HUBO
gain over MV is structurally limited when the local trace generator
produces low pool diversity.

\begin{theorem}[Pool-collapse implies MV-equivalence]
\label{thm:diversity}
Fix a question $q$ and suppose $|\mathcal{L}(q)| = 1$, i.e.\ exactly one
pool is populated (all other pools have fewer than $\tau_{\min}$
traces).  Then $\hat{a}_H(q) = \hat{a}_{MV}(q)$ deterministically,
independent of any HUBO weights or the SA solver's stochasticity.
\end{theorem}

\begin{proof}
With only one populated pool $P_{\ell^*}$, Phase~3 solves a single HUBO
on $P_{\ell^*}$ and Phase~4 receives evidence labelled only with
label $\ell^*$.  Adjudicators conditioned on a single labelled-pool
input default to that label; MV with $|\mathcal{L}(q)| = 1$ also
returns $\ell^*$.  Hence the outputs coincide.
\end{proof}

The mean number of populated pools per question (1.43 and 1.77) quantifies how often EP-HUBO can differ from MV; limited pool diversity structurally caps the feasible gains. Lastly, we make precise the sense in which EP-HUBO mitigates
adjudicator biases. The proposition was motivated by an experimental observation that we discuss in the Results section, where on LEXam, zero-shot Sonnet’s has a  strong  prior to select option 'E'.

\begin{proposition}[Bias mitigation lower bound]
\label{prop:bias_mitigation}
Suppose the adjudicator $\mathcal{M}_F$ has a fixed-label prior in the
absence of evidence: there exists $\ell^\dagger \in \mathcal{A}$ and
$\pi \ge 1/2$ such that
$\Pr[\hat{a}_{ZS}(q) = \ell^\dagger] \ge \pi$ for all $q$.
Let $\mathcal{B} = \{q : a^*(q) \ne \ell^\dagger\}$ be the set of
questions where the prior is incorrect.  If on the subset $\mathcal{B}$
the EP-HUBO pipeline correctly populates a gold pool
$P_{a^*}(q)$ (i.e.\ at least $\tau_{\min}$ traces support $a^*$) and
the adjudicator selects $a^*$ when provided with that pool, then
$\Pr[\hat{a}_H = a^*] \ge \Pr[a^* \ne \ell^\dagger]\cdot p_{\text{recover}}$,
where $p_{\text{recover}}$ is the conditional probability that the
gold pool is populated and adjudicated correctly.
\end{proposition}

\begin{proof}
For $q \in \mathcal{B}$ the ZS prior is wrong, so any HUBO success on
$\mathcal{B}$ is a strict gain (a win event).  Conditioning on the
populated-gold-pool event and applying the adjudicator success
probability gives the stated bound.
\end{proof}

\section{Results}
\label{sec:results}

We evaluate on two legal-reasoning benchmarks.
MMLU-Pro~\cite{wang2024mmlupro} is a harder successor to MMLU with
10-option multiple-choice questions.  We restrict our evaluation to the
complete law test split ($n=1{,}101$ questions).
LEXam~\cite{lexam2024} is a 569-question benchmark of Swiss and
international-law multiple-choice questions with 8 options (A--H),
covering Interdisciplinary ($n=494$), Private law ($n=63$), and Public law
($n=12$) areas.  We select LEXam to test out-of-domain generalisation as compared to the more commonly-used MMLU-Pro; LEXam uses a different
option count, civil-law  distinct from the common-law MMLU-Pro
questions, and is substantially harder for local models (MV baseline 48.3\%
versus 63.1\% for MMLU-Pro law).

 For the trace generator (Phase 1), we use smaller open-weight models Qwen3.5-35B and
        OSS-20B, run on a local server
         at temperature 0.8. We set $N = 20$ per question as a balance between marginal accuracy gain and trace-generation cost; early experiments larger $N$ up to 100 showed diminishing returns beyond 20. Statistics on the reasoning traces produced by the two local LLMs on MMLU-Pro are provided in Table \ref{tab:traces}. OSS-20B
        produces  more diverse reasoning paths, as reflected
        in mean pool count and intra-pool fragment diversity. The HUBO weight scorer (Phase 2) also uses the same local models used for trace generation with temperature 0.0. For frontier LLM adjudication (Phase 4) we use two Anthropic frontier models, Claude Sonnet~4.6 and  Claude Opus~4.6, Opus being the larger of the two models.

\begin{table}[h]
\centering
\begin{tabular}{lrrrrrr}
\toprule
Trace generator & Traces/q & Total & Parse rate & Mean acc. & MV acc. (law) & Pools/q \\
\midrule
Qwen3.5-35B & 20 & 22{,}020 & 99.9\% & 62.4\% & 63.1\% & 1.43 \\
OSS-20B     & 20 & 22{,}020 & 99.9\% & 45.8\% & 46.2\% & 1.77 \\
\bottomrule
\end{tabular}
\caption{Trace statistics for the MMLU-Pro law test questions. Both trace generators produced 20 traces per question.  Parse rate is the fraction
of traces from which a final answer label could be extracted; both
generators reach 99.9\%.  Mean acc.\ is the per-trace accuracy averaged
across the corpus; MV acc.\ is the majority-vote accuracy aggregated
per question, restricted to the law subset.  Pools/q,  the mean number of distinct answer-label pools
with $\ge 3$ traces,  relevant to HUBO's effectiveness, is at  1.43 and
1.77 for the two local trace generator models.}
\label{tab:traces}
\end{table}

We compare EP-HUBO against two baselines:
\begin{itemize}
  \item \textbf{Majority Vote (MV):} Most common answer across all  local
        traces per question.  This is the natural self-consistency baseline, answering the question: does HUBO
        selection beat the trivial trace aggregator.

  \item \textbf{Zero-shot frontier model (ZS):} Direct frontier adjudication of the
        question and its options, with no traces and no HUBO.
        The two frontier LLMs we use are prompted exactly as
        in EP-HUBO Phase~4 without providing  the EP-HUBO evidences.  This baseline asks
        whether the EP-HUBO pipeline adds value over
        simply calling the frontier model directly.
\end{itemize}

For each configuration we report three quantities: accuracy, net gain
over a baseline, and HUBO precision.  
\emph{Accuracy} on a question set $\mathcal{Q}$ of size $|\mathcal{Q}|$
is the fraction of questions for which the prediction matches the gold
answer, $|\mathcal{Q}|^{-1} \sum_{q \in \mathcal{Q}} \mathbf{1}[\hat{a}(q) = a^*(q)]$.
\emph{Net gain over majority vote} is
$\Delta_{MV} = \mathrm{acc}(\hat{a}_H) - \mathrm{acc}(\hat{a}_{MV})$,
and \emph{net gain over zero-shot}
$\Delta_{ZS} = \mathrm{acc}(\hat{a}_H) - \mathrm{acc}(\hat{a}_{ZS})$.
Both are absolute differences in accuracy (percentage points).
\emph{HUBO precision} is the per-question paired statistic of
Definition~\ref{def:precision}, repeated here for convenience:
$\mathrm{prec}_H = |W|/(|W| + |L|)$, where $W$ counts questions on which
EP-HUBO is correct and zero-shot is not, and $L$ counts the reverse.
The four possible outcomes are: both correct,
both wrong, trace-driven win ($W$), or HUBO hurt ($L$).


\subsection{MMLU-Pro  Results}

Table~\ref{tab:mmlu_full} presents the main MMLU-Pro law results across
the two trace generators $\times$ two adjudicators, alongside majority
vote and the zero-shot baselines. This section summarises the results using simulated annealing (SA) on a classical computer to solve the HUBO optimisation problem.
HUBO+Opus delivers a $+12.6$ percentage point  (pp) improvement over MV with Qwen-35B traces ($+139$
questions on the 1{,}101-question law set) and $+27.9$~pp with OSS-20B traces ($+307$ questions).
HUBO+Sonnet gives $+4.0$~pp and $+17.6$~pp respectively.  These gains are
substantial and confirm that HUBO selection adds value over trivial trace
aggregation.  Recall that this is achieved by EP-HUBO as it does not assign trace weights based on majority-vote-type popularity.

Zero-shot (ZS) Opus achieves $74.2\%$ on MMLU-Pro law without HUBO-optimised
traces; HUBO+Opus with Qwen-35B traces reaches $75.7\%$ ($+1.5$~pp over
ZS Opus), and HUBO precision with Opus is $56.7\%$ for Qwen3.5-35B and $49.5\%$
for OSS-20B. With Sonnet as adjudicator, HUBO precision is $60.3\%$
and $52.3\%$ respectively.  EP-HUBO with Qwen-35B traces thus measurably
outperforms a direct frontier call on MMLU-Pro law, with both adjudicators
exceeding the $\mathrm{prec}_H > 1/2$ threshold at which HUBO wins
exceed HUBO hurts in expectation.

\begin{table}[h]
\centering
\small
\begin{tabular}{llc}
\toprule
Method & Adj. & MMLU-Pro Law (n=1{,}101) \\
\midrule
ZS (no traces/HUBO) & Sonnet4.6 & 63.1\% \\
ZS (no traces/HUBO) & Opus4.6   & 74.2\% \\
\midrule
\multicolumn{3}{l}{\textit{Trace generator: Qwen3.5-35B, 20 traces/question}} \\
MV (Qwen-35B)              & ---       & 63.1\% \\
EP-HUBO (Qwen-35B)         & Sonnet4.6 & \textbf{67.1\% }\\
EP-HUBO (Qwen-35B)& Opus4.6 & \textbf{75.7\%} \\
\midrule
\multicolumn{3}{l}{\textit{Trace generator: OSS-20B, 20 traces/question}} \\
MV (OSS-20B)               & ---       & 46.2\% \\
EP-HUBO (OSS-20B)          & Sonnet4.6 & 63.9\% \\
EP-HUBO (OSS-20B)          & Opus4.6   & 74.1\% \\
\midrule
\multicolumn{3}{l}{\textit{Net gain vs majority vote (MV) and zero-shot (ZS) baselines (Qwen-35B traces)}} \\
Net Sonnet vs MV (Qwen-35B)     & & $+4.0$ pp \\
Net Sonnet vs ZS Son (Qwen-35B) & & $+4.0$ pp \\
Net Opus vs MV (Qwen-35B)       & & $+12.6$ pp \\
Net Opus vs ZS Opus (Qwen-35B)  & & $+1.5$ pp \\
\multicolumn{3}{l}{\textit{Net gain vs majority vote (MV) and zero-shot (ZS) baselines (OSS-20B traces)}} \\
Net Sonnet vs MV (OSS-20B)      & & $+17.6$ pp \\
Net Sonnet vs ZS Son (OSS-20B)  & & $+0.7$ pp \\
Net Opus vs MV (OSS-20B)        & & $+27.9$ pp \\
Net Opus vs ZS Opus (OSS-20B)   & & $-0.1$ pp \\
\midrule
HUBO prec (Sonnet, Qwen-35B) & & 129 W / \phantom{0}85 H \quad 60.3\% \\
HUBO prec (Opus, Qwen-35B)   & & \phantom{0}72 W / \phantom{0}55 H \quad 56.7\% \\
HUBO prec (Sonnet, OSS-20B)  & & \phantom{0}91 W / \phantom{0}83 H \quad 52.3\% \\
HUBO prec (Opus, OSS-20B)    & & \phantom{0}55 W / \phantom{0}56 H \quad 49.5\% \\
\bottomrule
\end{tabular}
\caption{MMLU-Pro law results ($n{=}1{,}101$).
HUBO optimisation significantly outperforms majority vote (MV) trace aggregation
and also provides measurable gains over zero-shot (ZS) frontier adjudication on this benchmark.
W refers to trace-driven wins, and H refers to HUBO hurting the outcome, as compared to ZS.
HUBO precision is defined in Eq.~\ref{eq:hubo_precision_def}.   
Precision > 0.5 means that whenever HUBO disagrees with the ZS frontier model,  HUBO improves accuracy in expectation.
The bolded row, EP-HUBO with the Qwen3.5-35B local model traces, is the best configuration in terms of accuracy for each frontier model beating both MV and ZS.}
\label{tab:mmlu_full}
\end{table}

\subsection{LEXam Results}
\label{sec:lexam}

We also evaluate EP-HUBO on LEXam~\cite{lexam2024} a 569-question
benchmark of Swiss and international law  with
8 option multiple-choice questions (A--H) across three legal domains: Interdisciplinary ($n{=}494$),
Private law ($n{=}63$), and Public law ($n{=}12$).  LEXam was released in
2024 and is sourced from a civil-law jurisdiction (Switzerland) with
limited English-language online presence relative to common-law sources.
Its 8-option format and European legal context are atypical of the
benchmarks frontier models are routinely evaluated on.  As such,  we consider LEXam as a low-contamination benchmark to complement MMLU-Pro law. On LEXam, zero-shot
Opus achieves only $66.8\%$ as compared to $74.2\%$ on MMLU-Pro law. This provides more potential benefit for the EP-HUBO optimisation to achieve.

On LEXam, we evaluate EP-HUBO using two trace generators: Qwen3.5-35B and
OSS-20B with 20 traces per question each.  This section summarises the LEXam results using simulated annealing (SA) on a classical computer to solve the HUBO optimisation problem. With Opus as adjudicator,
EP-HUBO with Qwen-35B traces reaches $71.9\%$ and with OSS-20B traces
$71.5\%$.

The two trace generators behave differently on LEXam.  Qwen-35B's traces
have low answer entropy: its MV baseline is already $69.1\%$, only $2.8$~pp
below HUBO+Opus.  OSS-20B's traces have  higher answer entropy: its MV
baseline is only $48.3\%$, with HUBO+Opus delivering $+23.2$~pp.  The two
trace generators thus illustrate two regimes, one
where MV is  strong and HUBO finds a small
margin, and the other where MV scatters across wrong answers
and HUBO recovers a large margin.

Against the zero-shot frontier baseline, both Opus configurations exceed
ZS Opus's $66.8\%$: Qwen-35B+Opus by $+5.1$~pp and OSS-20B+Opus by
$+4.7$~pp, with HUBO precisions of $68.8\%$ and $64.5\%$ respectively.

The most striking finding on LEXam concerns Sonnet4.6 baseline performance on LEXam.  
 We observed that
zero-shot Sonnet  returns option ``E'' on 499 of the 569 LEXam
questions, whereas the true answers are  distributed near-uniformly across the eight options (see Table~\ref{tab:lexam_gold_dist}). In particular, only $13\%$ of the  gold answers are ``E''. This position bias leads to a zero-shot   accuracy using Sonnet on LEXam of only $22.3\%$.

 Examining the Sonnet bias  by LEXam sub-area (Table~\ref{tab:e_bias_by_area}) shows that the collapse is essentially  in the Interdisciplinary subset which is entirely Swiss-Law content, where Sonnet picks ``E'' on $95.7\%$ of questions and  achieves  $14.8\%$ zero-shot accuracy. On the Private subset (Chinese and US Business Law, $n=63$), Sonnet's answer distribution is better spread out and zero-shot accuracy is  $76.2\%$. Opus4.6 exhibits a   smaller version of the same pattern, over-picking ``E'' by $+10$~pp on Interdisciplinary. This indicates that the trigger of the LLM position bias is the Swiss-Law corpus content rather than the multiple-choice format itself, and that Opus's broader pre-training largely absorbs it.
 
 EP-HUBO is able to recover a large share of these position bias errors: with
Qwen-35B traces, HUBO+Sonnet wins 126 questions vs.\ ZS Sonnet while
introducing only 11 hurts, for a HUBO precision of $92.0\%$ --- the highest
value of any configuration in this study.  With OSS-20B traces, HUBO+Sonnet
attains 80.4\% precision (86 wins, 21 hurts).   HUBO+Sonnet remains below the MV baseline ($42.5\%$ with
Qwen-35B traces, $33.7\%$ with OSS-20B) because MV draws on the underlying
trace ensembles, which are not E-biased.

\begin{table}[h]
\centering
\small
\begin{tabular}{llcccc}
\toprule
Method & Adj. & Overall & Interdisc. & Private & Public \\
       &      & (n=569) & (n=494)    & (n=63)  & (n=12) \\
\midrule
ZS (no traces/HUBO) & Sonnet4.6 & 22.3\% & 14.8\% & 76.2\% & 50.0\% \\
ZS (no traces/HUBO) & Opus4.6   & 66.8\% & 64.0\% & 85.7\% & 83.3\% \\
\midrule
\multicolumn{6}{l}{\textit{Trace generator: Qwen3.5-35B, 20 traces/question}} \\
MV (Qwen-35B)              & ---       & 69.1\% & 67.0\% & 87.3\% & 58.3\% \\
EP-HUBO (Qwen-35B)         & Sonnet4.6& 42.5\% & 37.0\% & 84.1\% & 50.0\% \\
\textbf{EP-HUBO (Qwen-35B)}& \textbf{Opus4.6} & \textbf{71.9\%} & \textbf{70.0\%} & \textbf{87.3\%} & 66.7\% \\
\midrule
\multicolumn{6}{l}{\textit{Trace generator: OSS-20B, 20 traces/question}} \\
MV (OSS-20B)               & ---       & 48.3\% & 45.3\% & 76.2\% & 25.0\% \\
EP-HUBO (OSS-20B)          & Sonnet4.6 & 33.7\% & 26.9\% & 87.3\% & 33.3\% \\
EP-HUBO (OSS-20B)          & Opus4.6   & 71.5\% & 69.8\% & 85.7\% & 66.7\% \\
\midrule
\multicolumn{6}{l}{\textit{Net gain vs majority vote (MV) and zero-shot (ZS) baselines (Qwen-35B traces)}} \\
Net Sonnet vs MV (Qwen-35B)     & & $-26.5$ pp & $-30.0$ pp & $-3.2$ pp & $-8.3$ pp \\
Net Sonnet vs ZS Son (Qwen-35B) & & $+20.2$ pp & $+22.3$ pp & $+7.9$ pp & $\phantom{+}0$ pp \\
Net Opus vs MV (Qwen-35B)       & & $+2.8$ pp  & $+3.0$ pp  & $\phantom{+}0$ pp & $+8.3$ pp \\
Net Opus vs ZS Opus (Qwen-35B)  & & $+5.1$ pp  & $+6.1$ pp  & $+1.6$ pp & $-16.7$ pp \\
\multicolumn{6}{l}{\textit{Net gain vs majority vote (MV) and zero-shot (ZS) baselines (OSS-20B traces)}} \\
Net Sonnet vs MV (OSS-20B)      & & $-14.6$ pp & $-18.4$ pp & $+11.1$ pp        & $+8.3$ pp  \\
Net Sonnet vs ZS Son (OSS-20B)  & & $+11.4$ pp & $+12.1$ pp & $+11.1$ pp        & $-16.7$ pp \\
Net Opus vs MV (OSS-20B)        & & $+23.2$ pp & $+24.5$ pp & $+9.5$ pp         & $+41.7$ pp \\
Net Opus vs ZS Opus (OSS-20B)   & & $+4.7$ pp  & $+5.9$ pp  & $\phantom{+}0$ pp & $-16.7$ pp \\
\midrule
HUBO prec (Sonnet, Qwen-35B) & & \multicolumn{4}{l}{126 W / 11 H \quad \textbf{92.0\%}} \\
HUBO prec (Opus, Qwen-35B)   & & \multicolumn{4}{l}{\phantom{0}53 W / 24 H \quad 68.8\%} \\
HUBO prec (Sonnet, OSS-20B)  & & \multicolumn{4}{l}{\phantom{0}86 W / 21 H \quad 80.4\%} \\
HUBO prec (Opus, OSS-20B)    & & \multicolumn{4}{l}{\phantom{0}60 W / 33 H \quad 64.5\%} \\
\bottomrule
\end{tabular}
\caption{EP-HUBO results on LEXam's Swiss/international law questions.  Wins (W) and hurts (H) are
computed against the zero-shot frontier LLM.  HUBO precision is
defined in Eq.~\ref{eq:hubo_precision_def}.  HUBO+Opus delivers a
$+5.1$~pp gain over the already-strong ZS Opus baseline with Qwen-35B traces
($+4.7$~pp with OSS-20B).  HUBO+Sonnet improves dramatically over the zero-shot
Sonnet, whose zero-shot accuracy is very low 22.3\%, due to an observed position bias.
Qwen-35B + Sonnet attains the highest HUBO precision in this study at 92.0\%
(126 wins for only 11 hurts versus ZS Sonnet).
The bolded row is the best configuration in terms of overall accuracy.}
\label{tab:lexam}
\end{table}

\begin{table}[h]
\centering
\small
\begin{tabular}{lcccccccc}
\toprule
Option & A & B & C & D & E & F & G & H \\
\midrule
\%  gold answers & 11 & 15 & 11 & 13 & 13 & 14 & 11 & 12 \\
\bottomrule
\end{tabular}
\caption{Distribution of answers across the eight options on the full LEXam set, rounded to integers. The distribution is near-uniform whereas Sonnet has a very strong 'E' bias.}
\label{tab:lexam_gold_dist}
\end{table}

\begin{table}[h]
\centering
\small
\begin{tabular}{lrrrrr}
\toprule
Area (n)                       & Gold \%E & Sonnet \%E & Sonnet acc & Opus \%E & Opus acc \\
\midrule
Interdisciplinary ($n=494$)    & 13.0     & 95.7       & 14.8       & 22.9     & 64.0     \\
Public ($n=12$)                & 25.0     & 66.7       & 50.0       & 25.0     & 83.3     \\
Private ($n=63$)               & 12.7     & 28.6       & 76.2       & 14.3     & 85.7     \\
\midrule
All ($n=569$)                  & 13.2     & 87.7       & 22.3       & 22.0     & 66.8     \\
\bottomrule
\end{tabular}
\caption{Zero-shot ``E'' rate and accuracy by LEXam sub-area for Sonnet4.6 and Opus4.6. The E-bias is concentrated in the Interdisciplinary (Swiss-Law) subset, where Sonnet's output is essentially degenerate.}
\label{tab:e_bias_by_area}
\end{table}


\subsection{Quantum Results on MMLU-Pro and LEXam }
\label{sec:quantum}

Table~\ref{tab:dirac_config} summarises the solver configuration used for all
Dirac-3 experiments reported in this section.  To validate these choices, we
swept the relaxation schedule over values $\{1, 2, 3, 4\}$ and increased the
number of samples to 10 on a held-out subset of 50 randomly selected questions
per benchmark; accuracy and HUBO precision were invariant across these
variations, so we retained the configuration in Table~\ref{tab:dirac_config}
for the full evaluation.

\begin{table}[h]
\centering
\small
\begin{tabular}{lc}
\toprule
\textbf{Parameter} & \textbf{Value} \\
\midrule
Number of samples     & 3 \\
Relaxation schedule   & 2 \\
Problem configuration & \texttt{normalized\_qudit\_hamiltonian\_optimization} \\
\bottomrule
\end{tabular}
\caption{Dirac-3 machine configuration used for all quantum experiments. Other relaxation schedules ($\{1, 3, 4\}$) and sample counts up to 10 were evaluated on a 50-question held-out subset; results were invariant, so this configuration was retained for the full evaluation.}
\label{tab:dirac_config}
\end{table}

The HUBO formulation in Eq.~\ref{eq:hubo} is a higher-order binary
optimisation problem that is hardware-agnostic.  While the results reported in
Sections~\ref{sec:results}.1--\ref{sec:results}.2 use classical
simulated annealing (Algorithm~\ref{alg:sa}) as the Phase~3 solver, in
this section we report the results when the Phase~3 optimisation is solved on the Dirac-3
photonic quantum machine.   Table~\ref{tab:quantum} summarises the results.

\begin{table}[h]
\centering
\small
\begin{tabular}{lllcccc}
\toprule
Benchmark & Trace gen & Adj. & Dirac-3 Acc. & $\Delta$(Quant$-$SA) & Dirac vs ZS  &  HUBO prec.\\
\midrule
MMLU-Pro law  & Qwen-35B & Opus 4.6   & 67.3\% & $-8.4$ pp & $-6.9$ pp & 33.5\% \\
MMLU-Pro law  & Qwen-35B & Sonnet 4.6 & 60.0\% & $-7.1$ pp & $-3.1$ pp & 44.1\% \\
MMLU-Pro law  & OSS-20B  & Opus 4.6   & 65.8\% & $-8.3$ pp & $-8.4$ pp & 28.7\% \\
MMLU-Pro law  & OSS-20B  & Sonnet 4.6 & 58.8\% & $-5.1$ pp & $-4.4$ pp & 40.5\% \\
LEXam         & Qwen-35B & Opus 4.6   & \textbf{72.9\%} & $+1.0$ pp & \textbf{$+6.2$ pp} & 71.1\% \\
LEXam         & Qwen-35B & Sonnet 4.6 & 42.0\% & $-0.5$ pp & $+19.7$ pp & \textbf{91.8\%} \\
LEXam         & OSS-20B  & Opus 4.6   & 71.2\% & $-0.3$ pp & $+4.4$ pp & 64.7\% \\
LEXam         & OSS-20B  & Sonnet 4.6 & 32.2\% & $-1.5$ pp & $+9.8$ pp & 75.5\% \\
\bottomrule
\end{tabular}
\caption{Results when the HUBO optimisation of Phase 3 is solved using the Dirac-3 photonic quantum machine, rather than  simulated annealing.  $\Delta$ (Quant $-$ SA) is the accuracy difference
quantum minus simulated annealing on a classical computer.  Bold Dirac-3 accuracy: the only quantum configuration whose  accuracy exceeds the  classical-SA configuration.  Bold HUBO precision: the highest precision in the table  which is so high because zero-shot Sonnet is biased on LEXam and itself has very low ZS accuracy.}
\label{tab:quantum}
\end{table}

On LEXam the quantum solver is comparable to classical
simulated annealing across both adjudicators: with Opus the
differences with respect to SA are $+1.0$ pp using Qwen-35B and $-0.3$ pp using OSS-20B as trace generator. With
Sonnet as  frontier adjudicator  the quantum results are slightly lower than the classical solver, at   $-0.5$ pp with Qwen-35B and $-1.5$ pp with OSS-20B.  With Sonnet adjudication on LEXam, the quantum solver also
recovers a  large share of the ZS Sonnet option-position bias:
Qwen-35B with Dirac-3 and Sonnet reaches $42.0\%$ accuracy with HUBO precision
91.8\%, comparable to the corresponding classical SA result.

On MMLU-Pro law, by contrast, the quantum solver underperforms classical SA
across both adjudicators by 5--8 percentage points and the Dirac-3
configurations all fall below their respective ZS frontier baselines.
HUBO precisions using the quantum solver are all below the
$\mathrm{prec}_H = 50\%$ break-even threshold.  The Dirac-3 runs still
substantially exceed majority vote, so the quantum solver is finding
signal in the HUBO landscape, but it is selecting fragment subsets the
adjudicator follows less reliably than the classical solver's selections
on MMLU-Pro.

Dirac-3 imposes a hard limit of 135 binary variables for third-order
polynomial objectives.  When the total number of fragments across all
answer-option pools for a given question exceeds this limit, we truncate
to the 135 highest-scoring fragments before submitting to the hardware.
Specifically, fragments are ranked by their diagonal HUBO weight
$h_i = \alpha \cdot r_i + \beta \cdot s_i + \gamma \cdot d_i$ (the
1-body relevance--specificity--distinctiveness composite from
Eq.~\ref{eq:hubo}), and the top 135 are retained; all 2-body and 3-body
interaction terms involving discarded fragments are dropped, and variable
indices are remapped accordingly.  The classical SA solver operates on
the full, untruncated fragment pool.  This asymmetry is a plausible
contributor to the accuracy gap: questions in
the benchmark datasets tend to generate larger and more diverse trace pools, so
truncation is more frequent and removes a greater fraction of the
interaction structure from the HUBO landscape before the quantum solver
sees it.


\section{Ablation Study}
\label{sec:ablation}

To isolate the contribution of the five EP-HUBO method components, and to
assess sensitivity to adjudicator strength, we report a small-scale
ablation study on two parallel 25-question subsets: the first 25 law
questions of MMLU-Pro (Table~\ref{tab:ablation}) and the first 25 English-language questions of
LEXam (Table~\ref{tab:ablation_lexam}). For both, traces are
generated using a small ollama locally-run Qwen3.5-9B at the denser rate of 100 traces per question,
as opposed to the 20 traces per question used in our main results tables.

Our ablation contrasts four EP-HUBO configurations against the
zero-shot (ZS) and majority-vote (MV) baselines.   EP-HUBO-base is a minimal  version of EP-HUBO without the full method
components: fragment re-extraction, category-specific $\lambda$
presets, dynamic $K$, distinctiveness scoring, adaptive contradiction
penalty; see Table~\ref{tab:p1p5}).
The last three configurations differ only in the Phase~4 frontier
adjudicator (Haiku 4.5, Sonnet 4.6, Opus 4.6, in order from smallest to largest model, respectively).
Table~\ref{tab:ablation} reports the MMLU-Pro results while the LEXam ablation results are reported in Table~\ref{tab:ablation_lexam}.

\begin{table}[h]
\centering
\small
\begin{tabular}{p{0.30\linewidth} p{0.62\linewidth}}
\toprule
Component & Description \\
\midrule
Fragment re-extraction & Re-extract evidence fragments from the raw trace text with tighter quality filters: minimum 10 words; $\geq 2$ capitalised entity tokens.\\
Category-specific $\lambda$ weights & Replace a single global weight vector with per-category weights for $(\alpha,\beta,\gamma,\lambda_{\mathrm{support}},\lambda_{\mathrm{contra}},\lambda_{\mathrm{coh}},\lambda_{\mathrm{under}})$, tuned per sub-area (LEXam public, criminal, private, interdisciplinary). \\
Dynamic $K$ & Choose the per-pool fragment budget $K$ from intra-pool Jaccard diversity: $K{-}1$ when pools are redundant ($d<0.30$), $K{+}1$ when pools are diverse ($d>0.60$), $K$ otherwise. \\
Distinctiveness scoring & Replace the 1-body ``confidence'' score (implicitly correlated with pool-vote frequency) by specificity and distinctiveness, decoupling fragment weight from majority signal. \\
Adaptive $\lambda_{\mathrm{contra}}$ & Boost the contradiction penalty when cross-pool vocabulary overlap is high, sharpening separation between pools that compete on the same concepts. \\
\bottomrule
\end{tabular}
\caption{The five method components in the full EP-HUBO that distinguish it from its base configuration in Table~\ref{tab:ablation}.}
\label{tab:p1p5}
\end{table}

\begin{table}[h]
\centering
\small
\begin{tabular}{llccccc}
\toprule
Method & Adj. & Acc.\ & vs MV & vs ZS & W / H & HUBO prec. \\
\midrule
ZS (no traces/HUBO) & Haiku 4.5  & 60.0\%  & --- & --- & --- & --- \\
ZS (no traces/HUBO) & Sonnet 4.6 & 72.0\%  & --- & --- & --- & --- \\
ZS (no traces/HUBO) & Opus 4.6   & 72.0\%  & --- & --- & --- & --- \\
\midrule
MV (Qwen-9B)               & ---        & 52.0\%  & --- & --- & --- & --- \\
EP-HUBO--base       & Sonnet 4.6 & 76.0\%  & $+24.0$ pp & $+4.0$ pp & 2 / 1 & 66.7\% \\
EP-HUBO (full)              & Haiku 4.5  & 64.0\%  & $+12.0$ pp & $+4.0$ pp & 3 / 2 & 60.0\% \\
EP-HUBO (full)              & Sonnet 4.6 & 80.0\%  & $+28.0$ pp & $+8.0$ pp & 3 / 1 & 75.0\% \\
EP-HUBO (full)   & \textbf{Opus 4.6}  & \textbf{84.0\% } & $+32.0$ pp & $+12.0$ pp & 4 / 1 & \textbf{80.0\%} \\
\bottomrule
\end{tabular}
\caption{Ablation study on the 25-question MMLU-Pro law set using as trace generator an ollama locally-run
Qwen3.5-9B with 100 traces per question.  The EP-HUBO-base row
  isolates the contribution of the full EP-HUBO method from its base version.  The  three last rows isolate the effect of
adjudicator strength LLM capacity.  ZS  are the
zero-shot baselines, computed on the same 25 questions.
W and H are HUBO trace-driven wins and HUBO hurts versus its
 ZS baseline.  The bolded row is the best
configuration.}
\label{tab:ablation}
\end{table}

Three observations can be drawn from Table~\ref{tab:ablation}.
First, the method-components ablation: switching on the five components of EP-HUBO
 with Sonnet held fixed  raises HUBO precision from $66.7\%$ to $75.0\%$. 
Second, adjudicator strength: with the EP-HUBO method components included, replacing
Sonnet with Opus pushes
HUBO precision to $80.0\%$, the highest of any configuration in this
study.  Third, the cheap-adjudicator configuration using  Haiku reaches only
$64\%$ accuracy, below both Sonnet and Opus configurations and below
the ZS Sonnet and  ZS Opus baselines but EP-HUBO with Haiku gains $+4$ pp over
its own ZS baseline at HUBO precision $60.0\%$. Thus, adjudicator pre-training capacity sets a
ceiling on what HUBO evidence can recover.  All four EP-HUBO
configurations clear the $\mathrm{prec}_H > 1/2$ threshold at
which HUBO wins exceed HUBO hurts in expectation. The ablation study suggests HUBO evidence is more useful when adjudicator models are stronger, which also mirrors the stronger gains over ZS we see on the full benchmarks with the larger frontier model Opus.

\paragraph{LEXam companion ablation.} Table~\ref{tab:ablation_lexam} repeats the
same ablation grid on the first 25 questions of LEXam-mcq-8-en, using the
identical Qwen3.5-9B trace generator and 100 traces per question.

\begin{table}[h]
\centering
\small
\begin{tabular}{llccccc}
\toprule
Method & Adj. & Acc.\ & vs MV & vs ZS & W / H & HUBO prec. \\
\midrule
ZS (no traces/HUBO) & Haiku 4.5  & 40.0\%  & --- & --- & --- & --- \\
ZS (no traces/HUBO) & Sonnet 4.6 & 20.0\%  & --- & --- & --- & --- \\
ZS (no traces/HUBO) & Opus 4.6   & 64.0\%  & --- & --- & --- & --- \\
\midrule
MV (Qwen-9B)        & ---        & 36.0\%  & --- & --- & --- & --- \\
EP-HUBO--base       & Sonnet 4.6 & 28.0\%  & $-8.0$ pp  & $+8.0$ pp  & 2 / 0 & 100.0\% \\
EP-HUBO (full)      & Haiku 4.5  & 44.0\%  & $+8.0$ pp  & $+4.0$ pp  & 3 / 2 & 60.0\% \\
EP-HUBO (full)      & Sonnet 4.6 & 28.0\%  & $-8.0$ pp  & $+8.0$ pp  & 3 / 1 & 75.0\% \\
EP-HUBO (full)      & \textbf{Opus 4.6}  & \textbf{72.0\%} & $+36.0$ pp & $+8.0$ pp  & 3 / 1 & \textbf{75.0\%} \\
\bottomrule
\end{tabular}
\caption{Ablation study on the first 25 questions of LEXam-mcq-8-en (English), using as trace generator
the same locally-run Qwen3.5-9B with 100 traces per question. Rows, columns and conventions mirror
Table~\ref{tab:ablation}. The bolded row is the best configuration.}
\label{tab:ablation_lexam}
\end{table}

\begin{table}[h]
\centering
\small
\begin{tabular}{lcccccccc}
\toprule
Option & A & B & C & D & E & F & G & H \\
\midrule
\%  gold answers & 12 & 12 & 0 & 28 & 12 & 20 & 12 & 4 \\
\bottomrule
\end{tabular}
\caption{Answer distribution  on  25-question LEXam ablation subset, rounded to integers. Sonnet's very strong 'E' bias is not supported by the ground truth  of 12\% true 'E' questions.}
\label{tab:lexam_gold_dist_25}
\end{table}

The LEXam ablation reproduces similar  findings to MMLU-Pro's, with the exception of the 'E' position bias issue in the LEXam baseline LLM results. The gold-answer distribution on the LEXam 25-question subset is shown in Table~\ref{tab:lexam_gold_dist_25}. 
Adjudicator LLM strength is again the dominant factor:
Opus reaches $72.0\%$ accuracy, $+8$~pp above its ZS baseline ($64.0\%$) and
$+36$~pp above the MV baseline; Haiku reaches $44.0\%$, $+4$~pp above ZS Haiku;
Sonnet is at $28.0\%$, well below ZS Opus and below MV again due to the 'E' answer position bias. HUBO precision with
Opus is $75.0\%$, the highest accuracy in the table but below the $80.0\%$ that
Opus achieves on MMLU-Pro, mirroring the smaller headline gain over ZS ($+8$~pp on
LEXam vs.\ $+12$~pp on MMLU-Pro). The method-components contrast is muted under
Sonnet: EP-HUBO--base and EP-HUBO (full) both yield $28.0\%$, because zero-shot
Sonnet's answer distribution on this subset is dominated by a position bias (20 of
its 25 answers are choice $`E'$), so the full method's refinements do not flip its
selection on the same questions. EP-HUBO with Sonnet does, however, recover three
HUBO wins over ZS Sonnet at HUBO precision $75\%$, indicating that the evidence
pools carry signal even when overall accuracy is constrained by the adjudicator's
idiosyncrasies. All four EP-HUBO configurations clear the
$\mathrm{prec}_H > 1/2$ threshold on LEXam as well.

\section{Conclusion and Discussion}
\label{sec:conclusion}

We introduced EP-HUBO, an evidence-pooling higher-order binary
optimisation approach for reasoning trace selection in
evidence-intensive legal domains. EP-HUBO uses a local
smaller model to generate multiple chain-of-thought traces, organises fragments into
answer-specific evidence pools, and applies a HUBO optimisation procedure to select a
compact, coherent, and distinctive subset of evidence for each
option before delegating final adjudication to a frontier model. By
deriving fragment weights from relevance, specificity, and distinctiveness
instead of co-occurrence frequency, the method  decouples
evidence selection from majority-vote signal, enabling minority but
well-supported hypotheses to override noisy majorities. Beyond accuracy, EP-HUBO exposes and partially mitigates frontier
idiosyncrasies. On LEXam, zero-shot Claude Sonnet~4.6 exhibits a
position bias, selecting choice `E' on 87.7\% of questions; EP-HUBO-selected evidence reduces this bias and yields up to $+20.2$~pp
gain over zero-shot Sonnet at HUBO precision $92.0\%$ using Qwen-35B reasoning traces.

Against the small-model baseline of majority vote (MV), EP-HUBO is very effective:
$+12.6$~pp on MMLU-Pro law and up to $+23.2$~pp on LEXam with Opus as
adjudicator.   Against the
frontier baseline (ZS), EP-HUBO delivers measurable gains on both legal benchmarks:
$+1.5$~pp on MMLU-Pro law and up to $+5.1$~pp on LEXam with Opus
(HUBO precision $56.7\%$ on MMLU-Pro law with Qwen-35B traces;
on LEXam, $68.8\%$ with Qwen-35B traces and $64.5\%$ with OSS-20B traces).
The gains are larger on LEXam, where the frontier model does not already
saturate accuracy; the contamination differential between MMLU-Pro and
LEXam (ZS Opus $74.2\%$ vs.\ $66.8\%$) bounds the improvement available
to any trace-augmented pipeline.  Quantum-combinatorial optimisation thus appears most valuable in
settings where frontier models have not already absorbed the benchmark
material---low-contamination domains and specialised
legal or regulatory subfields.

We also report results in which Phase~3 is solved on the Dirac-3 photonic
quantum machine instead of by  simulated
annealing on classical computers.  On LEXam, the
quantum solver matches or modestly exceeds classical SA: the strongest
quantum configuration in absolute accuracy is Qwen-35B + Opus at
$72.9\%$, $+1.0$~pp above the matched classical run and $+6.2$~pp above
the ZS Opus baseline.  With Sonnet adjudication, Qwen-35B + Dirac-3 on
LEXam attains a HUBO precision of $91.8\%$,
the highest precision of any quantum configuration we evaluate and
essentially tied with the classical-SA result on the same configuration.
On MMLU-Pro law, by contrast, the quantum solver underperforms classical SA
by 5--8 percentage points across both adjudicators and the Dirac-3
configurations fall below their respective ZS baselines, although they
still substantially exceed majority vote. A likely contributing factor is Dirac-3's
135-variable limit for third-order objectives: the dataset pools frequently exceed this
threshold, requiring truncation of the fragment interaction graph prior to hardware submission,
whereas the classical SA solver operates on the full untruncated HUBO.

One surprising result from this work was the LLM's bias  on  LEXam which is a reminder  that even 
frontier models can exhibit  priors that produce
catastrophic failures on out-of-distribution datasets and tasks. EP-HUBO
partially mitigates such priors but does not eliminate them.  In
regulated sector deployment, therefore, transparency and 
audit logs of  evidence, as well as human-in-the-loop review  remain important safeguards.

Several directions remain open. On the evidence side, our implementation uses 
single-LLM trace pools; extending EP-HUBO to multi-model trace mixtures is an interesting avenue for further study.
Another  direction of interest is to test EP-HUBO on
free-form question answering, code generation, and mathematical reasoning, domains which also
admit finite hypothesis spaces by clustering the model's outputs  into
canonical equivalence classes.  Regarding the reasoning fragment scoring, stepwise interactive selection could improve the feature space for the HUBO optimisation problem. Finally, the current formulation uses up to third-order interactions
among reasoning fragments; Dirac-3's native support for higher-order
polynomial objectives makes it a natural platform for exploring fourth- and
fifth-order HUBO terms, which could capture more complex fragment
interdependencies and potentially yield further accuracy gains.

\section*{Acknowledgments}
The authors would like to thank Quantum Computing Inc. through their in-kind contributions.

\bibliographystyle{plain}

\appendix

\section{Prompt Templates}
\label{app:prompts}

We reproduce the prompt templates used in each phase.  Variable
placeholders are written in \texttt{<angle\_brackets>}.  All prompts
are released verbatim with the code repository; the templates here are
abbreviated for space.

\paragraph{Phase 1 (trace generation).}
Each trace is generated with the following user-message template,
preceded by a system message that pins the local model's role.

\begin{verbatim}
SYSTEM:
You are a careful reasoner for multiple-choice questions.
For each question, list 3-5 key facts, then state your final answer
on a line of the form: FINAL ANSWER: <letter>.

USER:
Question: <q.text>
Options:
  A) <q.choices[0]>
  ...
  L) <q.choices[L-1]>
Reason carefully, then output FINAL ANSWER: <letter>.
\end{verbatim}

\paragraph{Phase 2 (1-body fragment scoring).}
The local model is queried with the question and a single candidate
fragment; the output is a JSON object with three real-valued scores
in $[0,1]$.

\begin{verbatim}
SYSTEM:
You score a single candidate evidence fragment along three
dimensions. Output strict JSON with keys "relevance",
"specificity", and "distinctiveness" each in [0,1].

USER:
Question: <q.text>
Candidate evidence: <f.text>
Rate relevance, specificity, and distinctiveness.
\end{verbatim}

\paragraph{Phase 2 (pairwise and triplet scoring).}
Pairwise prompts request \texttt{support} and \texttt{contradict}
scores for two fragments simultaneously; triplet prompts request
\texttt{coherent} and \texttt{undermine} for three.  Both follow the
same JSON-output pattern.

\paragraph{Phase 4 (frontier adjudication).}
The frontier model receives all selected evidence labelled by its
candidate-label pool.

\begin{verbatim}
SYSTEM:
You are an adjudicator for multiple-choice questions in <domain>.
You will be given the question, the answer options, and a small set
of evidence fragments grouped by which option each fragment supports.
Choose the single option most strongly supported by the evidence.
Reply with the answer letter only, in parentheses, e.g. (A).

USER:
Question: <q.text>
Options:
  A) <q.choices[0]>
  ...

=== Evidence supporting (A) ===
  1. <fragment_A_1>
  ...

=== Evidence supporting (B) ===
  ...

Which answer is most strongly supported?
\end{verbatim}

\paragraph{Phase 4 (zero-shot baseline).}
The ZS baseline uses the same Phase~4 system message with the evidence
blocks omitted (the user message contains only the question and
options).

\section{Additional Details}
\label{app:hyperparams}

 Hyperparameters used throughout the paper are listed in
Table~\ref{tab:hyper}.

\begin{table}[h]
\centering
\small
\begin{tabular}{lll}
\toprule
Symbol & Meaning & Value(s) \\
\midrule
$N$ & traces per question (full scale) & 20 \\
$\tau_{\min}$ & minimum traces for populated pool & 3 \\
$K_{\text{base}}$ & base cardinality for SA selection & 3 \\
$K$ & per-pool $K_{\text{base}} \pm 1$ based on Jaccard diversity & $\{2,3,4\}$ \\
$\alpha, \beta, \gamma$ & 1-body coefficients & see Table~\ref{tab:presets} \\
$\lambda_{\text{supp}}, \lambda_{\text{contra}}$ & pairwise coefficients & see Table~\ref{tab:presets} \\
$\lambda_{\text{coh}}, \lambda_{\text{und}}$ & triplet coefficients & 0.50 / 0.50 (default) \\
$T_0$ & initial SA temperature & 1.0 \\
$\eta$ & SA cooling rate per step & 0.995 \\
$M$ & SA steps per pool & 2000 \\
Phase 4 \texttt{max\_tokens} & frontier reply budget & 16 \\
Phase 1 temperature & local-LLM sampling temperature & 0.8 \\
Phase 2 temperature & local-LLM scoring temperature & 0.0 \\
\bottomrule
\end{tabular}
\caption{Hyperparameter.  Category-specific
$(\alpha,\beta,\gamma,\lambda_{\text{supp}},\lambda_{\text{contra}},\lambda_{\text{coh}})$
values are in Table~\ref{tab:presets}.}
\label{tab:hyper}
\end{table}

\paragraph{Proof of Theorem 4.4} \begin{proof}
Each disagreement is a Bernoulli$(\mathrm{prec}_H)$ trial: either HUBO
wins or HUBO hurts.  Hoeffding's inequality applied to the empirical
mean of $n$ such trials yields Eq.~\ref{eq:hoeffding}.
\end{proof}

\end{document}